\documentclass[conference]{IEEEtran}

\usepackage{amsmath,graphicx,amssymb,amsfonts,amsthm} 
\usepackage{cite}
\usepackage{enumerate}
\usepackage{bm}
\usepackage{color, colortbl}
\usepackage{balance}
    \usepackage{algorithmicx}
    \usepackage[ruled]{algorithm}
    \usepackage{algpseudocode}
    \usepackage{caption}
    \usepackage{subcaption}

\newtheorem{remark}{Remark}

\newtheorem{lemma}{Lemma} 

\algnewcommand\algorithmicinput{Input:}
\algnewcommand\INPUT{\item[\algorithmicinput]}
\algnewcommand\algorithmicoutput{Output:}
\algnewcommand\OUTPUT{\item[\algorithmicoutput]}

\def\T{{ \mathrm{\scriptscriptstyle T} }}

\title{Data-Driven Learning of the Number of States in Multi-State Autoregressive Models}

\author{Jie Ding, Mohammad Noshad, and Vahid Tarokh \\
School of Engineering and Applied Sciences, Harvard University, Cambridge, MA} 

\begin{document}


\maketitle

\begin{abstract}
In this work, we consider the class of multi-state autoregressive processes that can be used to model non-stationary time series of interest.
In order to capture different autoregressive (AR) states underlying an observed time series, it is crucial to select the appropriate number of states.
We propose a new and intuitive model selection technique based on the Gap statistics, which uses a null reference distribution on the stable AR filters to identify whether adding a new AR state significantly improves the performance of the model. 
To that end, we define a new distance measure between two AR filters based on the mean squared prediction error, and propose an efficient method to
generate  stable filters that are uniformly distributed in the coefficient space.
Numerical results are provided to evaluate the performance of the proposed approach.

\end{abstract}

\section{Introduction} \label{sec:introduction}
Modeling and forecasting time series is of  fundamental importance in various applications.
There may be occasional changes of behavior in a time series.
Some examples are the changes in the stock market due to the financial crisis, or the variations of an EEG signal caused by the mode change in the brain.
In the econometrics literature, this kind of time series is referred to 
as \textit{regime-switching} model \cite{goldfeld1973markov,Hamilton}.  
%
In regime switching models, the time series $\{x^{(n)}, n=1,2,\ldots \}$ is assumed to have $M$ states, and if $x^{(n)}$ belongs to state $m$ ($m=1,2,\ldots,M$), 
the probability density function (pdf) of $x^{(n)}$ conditioning on its past is in the form of $f_m(x^{(n)}|x^{(n-1)},\ldots,x^{(1)})$. 
The autoregressive (AR) model, one of the commonly used techniques to model stationary time series \cite{Hamilton}, is usually used to model each state. 
The autoregression of state $m$ is given by $x^{(n)}  +  \bm \gamma_{m}^\T \bm x^{(n)} = \varepsilon^{(n)}$ where $\varepsilon^{(n)}$ are independent and identically distributed (i.i.d.) noises with zero mean and variance $ \sigma^2_{m}$. 
Here $\bm x^{(n)} = [1, x^{(n-1)}$,  $\ldots$ $,x^{(n-L)}]^\T$, $\bm \gamma_m = [\gamma_{m0}, \gamma_{m1}, \ldots,\gamma_{mL}]^\T $ 
is a real-valued vector of length $L+1$ that characterizes state $m$.
A more detailed survey on this model can be found in \cite{hamilton2008regime}.
We refer to this model as a multi-state AR model and to $\bm \gamma_m$ as the AR filter or AR coefficients of state $m$.
%
%
The above model with $\bm \gamma_{m} = 1$ was first analyzed by Lindgren \cite{lindgren1978markov} and Baum et al. \cite{baum1970maximization}.
The model with general $\bm \gamma_{m} $ is widely studied in the speech recognition literature \cite{poritz1982linear}. 
%
The multi-state AR model is a general statistical model that can be used to fit data in many real world applications.
It was shown that the model is capable of representing non-linear and non-stationary time series with multimodal conditional distributions and with heteroscedasticity  \cite{AR-Mixture-00}.
There are two basic underlying assumptions in this model: 1. Autoregression assumption, which is reasonable if the observations are obtained sequentially in time; 2. Multi-state assumption, which is reasonable if the stochastic process exhibits different behaviors in different time epochs.
For example, stock prices may have dramatic while not permanent changes in the case of business cycles or financial crises, and those dynamics can be described by stochastic transitions among different states.

Despite the wide applications of the multi-state AR model, there are few results on how to estimate the number of states $M$ in a time series.
Obviously, different values of $M$ produce a nested family of models and models with larger $M$'s fit the observed data better.
The drawback of using complex models with a large $M$ is the over-fitting problem which decreases the predictive power of the model.
Hence, a proper model selection procedure that identifies the appropriate number of states is vital.
It is tempting to test the null hypothesis that there are $M$ states against the alternative of $M+1$.
Unfortunately, the likelihood ratio test of this hypothesis fails to satisfy the usual regularity conditions since some parameters of the model are unidentified under the null hypothesis.
An alternative is to apply Akaike information criterion (AIC) \cite{AIC-73} or Bayesian information criterion (BIC) \cite{BIC-78} to introduce a penalty on the complexity of the model in the model selection procedure. 
However, in general AIC and BIC are shown to be inaccurate  in estimating the number of states~\cite{Problem-of-AIC-BIC-96}.

In this paper, we propose a model selection criterion inspired by the work of Tibshirani et al. \cite{Gap-Statistics-01} who studied the clustering of i.i.d. points under Euclidean distance. 
The idea is to identify $M$ by \textit{comparing the goodness of fit for the observed data with its expected value under a null reference distribution}.
To that end, we first draw a \textit{reference curve} which plots the ``goodness of fit'' versus $M$ based on the most non-informative distributed data, and describes how much adding new AR states improves the goodness of fit. We then draw a similar curve based on the observed data. In this work we choose the ``goodness of fit'' measure to be the mean squared prediction error (MSPE). Finally, the point at which the gap between the two curves is maximized is chosen as the estimated $M$.
%

Besides the simplicity and effectiveness, another benefit of the proposed model selection criterion is that it is adaptive to the underlying characteristics of AR processes. The criterion for the processes of little dependency, i.e., the roots of whose characteristic polynomial are small, is different from the criterion for those of large dependency. 
In this sense, it takes into account the 
characteristics behind the observed data in an unsupervised manner, even though no domain knowledge or prior information is given.

The remainder of the paper is outlined below. 
In Section~\ref{sec:gap}, we propose the Gap statistics for estimating the number of AR states in a time series.
Section~\ref{sec:model} formulates a specific class of the multi-state AR model, where the transitions between the states are assumed to be a first order Markov process.
We emphasize that this parametric model is considered primarily for simplicity and the proposed Gap statistics can be applied to general multi-state AR processes.
A new initialization approach is also proposed that can effectively reduce the impact of a bad initialization on the performance of the expectation-maximization (EM) algorithm.
Section~\ref{sec:performance} presents some numerical results to evaluate the performance of the proposed approach.
Experiments show that the accuracy of the proposed approach in estimating the number of AR states surpasses those of AIC and BIC.

\section{Gap Statistics} \label{sec:gap}

This section describes our proposed criterion for selecting the number of states in a multi AR process, inspired by \cite{Gap-Statistics-01}.
We draw a reference curve, which is the expected value of MSPE under a null reference distribution versus $M$, and use its difference with the MSPE of the observed data to identify the number of states, $M$. 
We show that computing each point of the reference curve 
turns out to be a clustering problem in the space of AR coefficients of a fixed size,
where the  distance measure for clustering is derived from the increase in MSPE when a wrong model is specified.
We derive the distance measure in closed form, introduce an approach to generate stable AR filters that are uniformly distributed, and apply the $k$-medoids algorithm to approximate the optimal solution for the clustering problem.
We first outline our proposed model selection criterion in Subsection~\ref{subsection:gap}, and then elaborate on the distance measure in Subsections~\ref{subsection:distance} and the generation of random AR filters in Subsections~\ref{subsection:distribution}.

\vspace{0.3cm}
\subsection{The Model Selection Criterion} \label{subsection:gap}

We use superscript $(n)$ to represent the data at time step $n$, and $\mathcal{N}(\mu,\sigma^2)$ to denote the normal distribution with mean $\mu$ and variance $\sigma^2$.
Symbols in bold face represent vectors or matrices. 
We start from a simple scenario where the data $\{ x^{(n)}, n=1,2,\ldots\}$ is generated using a single stable AR filter $\bm \psi_A$:
$
        x^{(n)} = - \bm \psi{_A^\T} \bm  x^{(n)}  + \varepsilon^{(n)}
$,
where $\bm x^{(n)} = [1, x^{(n-1)}$,  $\ldots$ $,x^{(n-L)}]^\T$, $\bm \psi_A = [\psi_{A0}, \psi_{A1}, \ldots,\psi_{AL}]^\T $, and $\varepsilon^{(n)} $ are  i.i.d. $\mathcal{N} (0, \sigma_A^2)$.
Suppose we are at time step $n-1$ and we want to predict the value at time $n$. If $\hat{x}^{(n)} = -\bm \psi{_A^\T} \bm x^{(n)}$ is used for prediction, the MSPE is
$ E \{ (x^{(n)} + \bm \psi{_A^\T} \bm x^{(n)} )^2\} = E\{(\varepsilon^{(n)})^2\}=\sigma_A^2$.
But if another AR filter is used for prediction instead of $\bm \psi_A$, i.e., $\hat{x}^{(n)} = - \bm \psi_B^\T \bm x^{(n)}$, the MSPE becomes $E\{ (x^{(n)} + \bm \psi_B^\T \bm x^{(n)} )^2 \}$.
The difference of the two MSPE is defined by 
\begin{align} 
D(\bm \psi_A , \bm \psi_B) 
  &= E\left\{ \left[ x^{(n)} + \bm \psi_B^\T \bm x^{(n)} \right]^2 \right\} - \sigma_A^2 \nonumber  \\ 
  &=E \left\{ \left[ (\bm \psi_A - \bm \psi_B)^\T \bm x^{(n)}\right]^2 \right\} . \label{mismatch}
\end{align}
It is easy to observe that $D(\bm \psi_A , \bm \psi_B)$ is always nonnegative, which means that using the mismatch filter for prediction increases MSPE.
We refer to $D(\bm \psi_A , \bm \psi_B)$ as the mismatch distance between two filters $\bm \psi_A$ and $\bm \psi_B$, though it is not a metric. 
When the data generated from $\bm \psi_A $ has zero mean, i.e., $\psi_{A0}  = 0$, we let $\bm \psi_A$ also represents $ [\psi_{A1}, \ldots, \psi_{AL}]^\T$ of length $L$ (with constant term omitted) with a slight abuse of notation, and we use $\bm \psi_B$ in the same manner.

As has been mentioned in Section~\ref{sec:introduction}, our model selection criterion is based on a reference curve that describes how much adding a new state increases the goodness of fit in the most non-informative or the ``worst'' case. To that end, we consider an $M$-state zero mean AR process  where at each time step $n$, nature chooses random mismatch filters (with zero constants) for prediction. 
In such a worst scenario, the $M$ filters that minimize the average mismatch distances to the random filters are naturally believed to be the true data generating filters, 
and that minimal value, which is the average MSPE, is plotted as the reference curve. 
This leads to the following clustering problem in the space of stable AR filters $R_L(r) \subset \mathbb{R}^L$, where
\begin{align*}
R_L(r) =&\{[\lambda_{1}, \ldots, \lambda_{L} ]^\T \mid z^L + \sum_{\ell=1}^{L}\lambda_{\ell} z^{L-\ell} = \prod\limits_{\ell=1}^L (z-a_{\ell}) , \\ 
&\lambda_{\ell} \in \mathbb{R}, |a_{\ell}| < r, \, 0< r \leq 1, \ell=1,\ldots,L \}.
\end{align*}
\\
{\bf Clustering of Stable Filters}: For a fixed $M$, let $\mathfrak{F} = \{\bm \psi_1$, $\bm \psi_2$, $\ldots$, $\bm \psi_F\}$ be a set of uniformly generated stable filters of a given length $L$. We cluster $\mathfrak{F}$ into $M$ disjoint clusters $C_1,C_2,\ldots,C_M$, and define the within cluster sum of distances to be
\begin{align} \label{objective_filter}
      W_M =   \min\limits_{\bm \gamma_1,\ldots,\bm \gamma_M } \left\{  \frac{1}{F} \sum_{m=1}^M \sum_{\bm \psi \in C_m} D(\bm \gamma_m, \bm \psi) \right\} +1 ,
 \end{align}
where $D(\bm \gamma_m, \bm \psi)$ is defined in (\ref{mismatch}) and will be further simplified in (\ref{dist2}), (\ref{dist3}) and (\ref{chicago}).
By computing $\log(W_M)$ for $M=1,\ldots,M_{\text{max}}$, we obtain the reference curve.
The optimization problem (\ref{objective_filter}) can be solved by the $k$-medoids algorithm \cite{kaufman1987clustering}.
%

The model selection criterion is outlined in Table~\ref{algo:gap}.
We note that the bound for the roots $0<r \leq 1$ is determined by the estimated filters, and thus the reference is data-dependent.
Intuitively, if the process has less dependency, or in other words a point has less influence on its future points, the roots of the characteristic polynomials of each AR process are closer to zero and the MSPE curve will have smaller values. 
Thus, the filters from which the reference curve is calculated should also be drawn from a smaller bounded space.

\alglanguage{pseudocode}
\begin{algorithm*} 
\vspace{0.4 cm}
\small
\caption{Model Selection Based on Gap Statistics}
\label{algo:gap}
\begin{algorithmic}[1]
\INPUT $\{ x^{(n)}, n=1,\ldots , N\}$, $M_{\textrm{max}} $ (which is assumed to contain the ``correct'' number of states)
\OUTPUT  The number of AR states $M_{\text{opt}}$.
    \For {$M = 1 \to M_{\textrm{max}} $ }
        \State  Fit a multi-state AR model to the data
        (for instance using the EM algorithm described in Algorithm~\ref{algo:EM} ) 
    \State  Compute the MSPE $\hat{W}_M$ based on the estimated model.
   \EndFor
   \State Plot $\log(\hat{W}_M), M=1,\ldots,M_{\textrm{max}} $, referred to as the ``observed MSPE curve''
   \State Compute the largest absolute value of the roots of each estimated AR filter for the case $M=M_{\textrm{max}}$, denoted by $r_1,\ldots,r_{M_{\textrm{max}}}$. Let
    		$r=\min\{ \max\{r_1,\ldots,r_{M_{\textrm{max}}}\}, 1\}.$
   %
   \For {$\ell = 1 \to Iter$ (number of iterations) } 
   	\State Run Algorithm~\ref{algo:gen} (to be introduced in Subsection~\ref{subsection:distribution}) to generate $F$ (e.g. $F=N$) independent and uniformly distributed stable filters $\mathfrak{F} = \{\bm \psi_1, \ldots, \bm \psi_F\}$ from $R_{L}(r)$.
   	\For {$M = 1 \to M_{\textrm{max}} $ }
   		\State Run Algorithm~\ref{algo:cluster} 
		to approximate the optimum of (\ref{objective_filter}), and obtain $\log(W_{M\ell}),M=1,\ldots,M_{\textrm{max}}$.
   	\EndFor
    \EndFor
   \State Let $W_M = \sum_{\ell=1}^{Iter} W_{M\ell}/Iter.$  Plot $\log(W_M), M=1,\ldots,M_{\textrm{max}}$ as the reference curve (see Fig.~\ref{Ref_Curve} for an example).
   \State Choose $M_{\text{opt}}$ to be the smallest $M \ (1 \leq M < M_{\textrm{max}})$ that satisfies $ \log(W_M)-\log(\hat{W}_M) \geq \log(W_{M+1}) - \log(\hat{W}_{M+1})$ if there exists any; otherwise  $M_{\text{opt}}=M_{\textrm{max}}$. 
\end{algorithmic}
  \vspace{0.2cm}%
\end{algorithm*}

  \vspace{0.4 cm}
\alglanguage{pseudocode}
\begin{algorithm*}
  \vspace{0.4 cm}
\small
\caption{Clustering Stable AR filters via ``$k$-medoids'' Algorithm}
\label{algo:cluster}
\begin{algorithmic}[1]
\INPUT  A set of stable filters $\mathfrak{F} = \{\bm \psi_1$, $\ldots$, $\bm \psi_F\}$, the number of desired clusters $M$, a  number $0<\delta<1$ (used for the stopping criterion).
\OUTPUT The minimum within-cluster sum of distances (WCSD) $w_{\ell}$ and $\{\bm \psi_{c_1},\ldots, \bm \psi_{c_M} \} \subset \mathfrak{F} $ that approximate the $M$ centers.
	\State Generate a matrix $\bm D_{F\times F}$ whose elements are pairwise distances between filters: $D_{uv} =D(\bm \psi_u , \bm \psi_v) $.
	\State Initialize $M$ clusters characterized by centers
		$c_m$ and associated  sets of indices $ I_m $ ($m=1,\ldots,M$) that form a partition of $\{1,\ldots,F\}$.
	\State Compute  $w_{1} = \sum_{m=1}^M \sum_{ u  \in I_m} D(\bm \psi_{c_m}, \bm \psi_u)$.
		Let $w_{0} = 2w_1/(1-\delta) ,  \ell=1$ (for initialization purpose).
	\While { $ w_{\ell-1}-w_{\ell} > \delta w_{\ell-1} $ }
		\State $\ell = \ell+1$, $w_{\ell} = w_{\ell -1 }$.
		\For {$m = 1 \to M$ }
		\State Suppose that $I_m = \{ I_m[1], \ldots, I_m[i_m] \}$ and let $ k=1$. 
			\While{$k < i_m$ } 
				\State Consider the candidates for the new centers, $\hat{c}_1, \ldots, \hat{c}_M$, where $\hat{c}_{m'} = c_{m'} \ \ (m' = 1 , \ldots , M, \ m' \neq m)$ and $\hat{c}_m = I_m[k]$. 
				\State For each $u=1,\ldots,F$,  
		        		let  $u \in \hat{I}_{m'}$ if $D(\bm \psi_{\hat{c}_{m'}}, \bm \psi_u) \leq D(\bm \psi_{\hat{c}_{j}}, \bm \psi_u)  \ \ (j = 1,\ldots, M, \ j \neq m')$.
				\State Compute the WCSD given the new clusters: $\hat{w}_{\ell} = \sum_{m'=1}^M \sum_{ u  \in \hat{I}_{m'}} D(\bm \psi_{\hat{c}_{m'}}, \bm \psi_u)$.
				\If {$ \hat{w}_{\ell} < w_{\ell} $}
					\State $k=1$,
						  $w_{\ell} = \hat{w}_{\ell}$,
						  $c_m = I_m[k]$,
						  $I_{m'} = \hat{I}_{m'} \ \ (m' = 1 , \ldots , M)$.
				\Else
					\State $k = k+1$.
				\EndIf
			\EndWhile
		\EndFor
  	 \EndWhile
\end{algorithmic}
\vspace{0.4cm}%
\end{algorithm*}

\subsection{Distance Measure for Autoregressive Processes} \label{subsection:distance}

In this subsection, we provide the explicit formula for the distance in Equation (\ref{mismatch}).
Assume that the data is generated by a stable filter $\bm \psi_A$ of length $L$.
Let $\Psi_A(z) = \sum\limits_{\ell=1}^{L}\psi_{A\ell} z^{-\ell}$ be the characteristic polynomial of $\bm \psi_A$, and let $a_1, \ldots, a_{L}$ denote the roots of $1+\Psi_A(z)$, i.e.,
        $1+ \Psi_A(z) = \prod\limits_{\ell=1}^{L} \left( 1-a_{\ell} / z \right)$,
where $a_1, \ldots, a_{L}$ lie inside the unit circle ($|a_{\ell}| < 1$).
Similarly define $\Psi_B(z), b_1,\ldots,b_{L}$ for $\bm \psi_B$.
The value in  (\ref{mismatch})
can be computed using the power spectral density and Cauchy's integral theorem as:
    \begin{align}
	D\left(\bm \psi_A , \bm \psi_B \right) &=
	D_0\left(\bm \psi_A , \bm \psi_B \right)
	+ \left( \frac{1+\sum\limits_{\ell=1}^L \psi_{B\ell}}{1+ \sum\limits_{\ell=1}^L \psi_{A\ell}}\psi_{A0} + \psi_{B0} \right)^2 \label{dist1} 
    \end{align}
where $D_0 \left(\bm \psi_A , \bm \psi_B \right) =$
    \begin{align}	
	&\frac{\sigma_A^2}{2 \pi} \int_{-\pi}^{\pi} \frac{\left| \Psi_A (e^{j\omega} )-\Psi_B (e^{j\omega} )\right|^2}{ \left|1+ \Psi_A (e^{j\omega} )\right|^2}  d\omega \nonumber \\
	&=
	\sigma_A^2 \sum_{k=1}^{L} \frac{\prod\limits_{\ell=1}^{L} (a_k-b_{\ell})}{a_k \prod\limits_{\substack{\ell=1\\ \ell \neq k}}^{L} (a_k-a_{\ell})} \left( \frac{\prod\limits_{\ell=1}^{L} (1- a_k b_{\ell}^{*} )}{\prod\limits_{\ell=1}^{L} (1- a_k a_{\ell}^{*} )} -1 \right), \label{dist2}
    \end{align}
for $a_k\neq 0, a_k \neq a_{\ell}, k \neq \ell$, where $a^{*}$ denotes the complex conjugate of $a$.
For the degenerate cases when $a_k=0$ or $a_k = a_{\ell}$, $D(\bm \psi_A , \bm \psi_B )$ reduces to $\lim_{a_k \rightarrow 0} D(\bm \psi_A , \bm \psi_B )$
or $\lim_{a_k \rightarrow a_{\ell}} D(\bm \psi_A , \bm \psi_B )$.
\begin{remark}
For now we assume that $x^{(n)}$ at each state has zero mean by default, unless explicitly pointed out. We use $D_0(\cdot)$ in Identity (\ref{dist2}) instead of $D(\cdot)$ in Identity (\ref{dist1}) to compute the reference curve. 
The derived reference curve can be applied to the general case.
The reason is that it is more difficult to detect two AR states with the same mean than those that have different means.
Therefore, the reference curves for the zero mean case (the ``worst'' case) can be used in general.


The distance measure defined in Equation (\ref{dist2}) is proportional to $\sigma_A^2$. We consider $\sigma_A^2=\sigma^2$ which results in a constant $\log \sigma^2$ in the computation of $\log W_M$ in (\ref{objective_filter}). Since it is the same for different $M$'s, we set $\sigma^2=1$ without loss of generality.
\end{remark}
The distance between two AR filters can be explicitly expressed in terms of the coefficients.
This is computationally desirable if the filters are random samples generated in the coefficient domain, as will be discussed in Subsection~\ref{subsection:distribution}.

\begin{figure*}[!t]
\begin{align}
D_{0}(\bm \psi_A, \bm \psi_B) =
 & \frac{ Po\Big(p_A(z), p_B\overline{p_B}(z)\Big) S([u_1,\ldots,u_{L-1}], 0) -
 	Po \Big(p_A(z), p_B\overline{p_B} {p'_A}\overline{p_A}(z) S([u_1,\ldots,u_{L-2}], {p'_A}\overline{p_A}(z))\Big)
	}{ Res\Big( p_A(z), {p'_A}\overline{p_A}(z)\Big) }  \nonumber \\
	&- \frac{ Po\Big(p_A(z), p_B(z)\Big) S( [ v_1,\ldots,v_{L-1}], 0) -
 	Po\Big(p_A(z), p_B{p'_A}(z) S( [ v_1,\ldots,v_{L-2}], {p'_A}(z) ) \Big)
	}{ Res\Big( p_A(z), {p'_A}(z)\Big) }  \label{dist3}
\end{align}
   \vspace*{-0.2 in}
\end{figure*}

{\bf Notations: } Consider two polynomials of nonnegative powers $p(z)$ and $q(z)$ respectively of degrees $u>0$ and $v>0$. 
Let $\overline{ q } (z), pq(z)$ respectively denote  the reciprocal polynomial of $q(z)$, and the multiplication of $p(z)$ and $q(z)$, i.e., $\overline{ q} (z) = z^{v} q(z^{-1})$, $ pq(z)=p(z)q(z)$.
Let $Res(p(z), q(z))$ be the resultant of $p(z)$ and $q(z)$.
Define $Po(p(z), q(z)) = \sum_{k=1}^{u} q(a_{k})$ and  $Po(p(z), 0)=0$, where $a_{1}, \ldots,a_u$ are the roots of $p(z)$.

\begin{lemma} \label{lemma:simple}
The values of $Res(p(z), q(z))$ and  $Po(p(z), q(z))$ can be computed as polynomials of the coefficients of $p(z)$ and $q(z)$.
\end{lemma}
The proof follows from the fact that the resultant of $p(z)$ and $q(z)$ is given by the determinant of their associated Sylvester matrix \cite{sturmfels1998introduction}, and that for any $n \in \mathbb{N}$, $\sum_{k=1}^u a_{k}^n$ can be  computed as polynomials in the coefficients of $p(z)$ via Newton's identities.
We further provide the following result. 
\begin{lemma}
Let $p_A(z)=z^L (1+\Psi_A(z) )=\prod_{\ell=1}^L (z-a_{\ell}), 
p_B(z)=z^L (1+\Psi_B(z) )=\prod_{\ell=1}^L (z-b_{\ell})$, $p'_A(z) = \partial (z p_A(z))/\partial z$. The value of $D_{0}(\bm \psi_A, \bm \psi_B)$ in Equation (\ref{dist2}) (with $\sigma_A=1$) can be computed in terms of the coefficients of $\bm \psi_A$ and $\bm \psi_B$ as in Equation (\ref{dist3}) (on the top of the next page),
where  $u_i = Po\Big( p_A(z), ({p'_A}\overline{p_A}(z))^i \Big)$, $v_i = Po\Big( p_A(z), ({p'_A}(z))^i \Big), i=1,\ldots,L-1$,
and the function $S(\cdot, \cdot)$ is defined as  $S([s_1, \ldots, s_h], t) =$
\begin{align*} 
\frac{1}{h!}
\det
\begin{pmatrix}
	s_1 - t & 1 &  0  & \cdots & 0\\
	s_2  - t^2 & s_1 -t & 2 & \ddots & \vdots \\
	\vdots & \vdots & \ddots & \ddots  & 0\\
	s_{h-1} - t^{h-2} & \vdots & \ddots & \ddots & h-1 \\
	s_{h} -t^{h} & s_{h-1} -t^{h-1} & \cdots & s_2-t^2 & s_1 -t
\end{pmatrix}
\end{align*}
for $h>0$, $S(\cdot, \cdot)=1$ for $h=0$, and $S(\cdot, \cdot)=0$ for $h<0$, where  $\det(\cdot)$ denotes the determinant of a square matrix.

\end{lemma}

Another simple way to compute the distance measure is given by the following lemma. 
\begin{lemma}
Let $\bm \Psi_A = [\psi_1,\ldots,\psi_L]^{\textrm{T}}$ be the true filter of an autoregression with zero mean.
The variance $\gamma_0$, the correlations  $\rho_k=\rho_{-k} \ (k=1,\ldots,L)$, and the covariance matrix $\bm \Gamma$ of the autoregression are respectively defined to be
$
\gamma_0 = E\left\{ (x^{(n)})^2\right\}, \ \rho_k=\rho_{-k}=E (x^{(n)} x^{(n-k)}) , \ \bm \Gamma =  \ [ \gamma_0 \rho_{i-j}]_{i,j=1}^{L} .
$
Define $\bm \rho = [\rho_1, \ldots, \rho_L]^{\textrm{T}}$, $\psi_{k} = 0$ for $k\leq 0$ and $k>L$,  $\delta_{i,j} = 1$ if $i=j $ and $\delta_{i,j} = 0$ otherwise $ (1 \leq i , j \leq L)$. 
Then $\bm \rho$ and $\gamma_0$ can be computed by  
$$
\bm \rho = - \bm  \Phi^{-1} \bm \psi_A, \quad \gamma_0 = (1+\bm \rho^{\textrm{T}}  \bm \psi_A)^{-1} ,
$$
where
 $\bm \Phi = [\Phi_{i,j}]_{ 1 \leq i,j \leq L}$ is determined by $\Phi_{i,j} = \psi_{i+j} + \psi_{i-j} + \delta_{i,j}$.
The value of $D_0\left(\bm \psi_A , \bm \psi_B \right)$  in terms of $\bm \psi_A$ and $\bm \psi_B$ can be computed by
\begin{align}
D_0\left(\bm \psi_A , \bm \psi_B \right) = (\bm \psi_A - \bm \psi_B)^{\text{T}} \bm \Gamma (\bm \psi_A - \bm \psi_B) . \label{chicago}
\end{align}
\end{lemma}

\subsection{Generating Uniformly Distributed Filters with Bounded Roots} \label{subsection:distribution}

As mentioned before, Gap statistics requires a reference curve that is calculated by clustering the filters randomly chosen from a reference distribution. In some scenarios we need to generate sample filters from $R_L(r)$, where $r$ is calculated from the observed data.
%
Inspired by the work of Beadle and Djuri\'{c} \cite{Uniform-Polynomial-Gen-97}, 
we provide the following result on how to generate a random point in $R_L(r)$ with uniform distribution
\begin{lemma} \label{lemma:distribution}
Generation of an independent uniform sample of $[\lambda_{1,L}, \ldots, \lambda_{L,L}]^\T \in R_L(r)$ can be achieved by the following procedure: \\ 
1.  Draw $\lambda_{1,1}$ uniformly on the interval $[-r,r]$; \\ 
2. For $k=2,\ldots, L$, suppose that we have obtained $[\lambda_{1,k-1}, \ldots, \lambda_{k-1,k-1}]^\T$ that is uniformly distributed in $R_{k-1}(r)$. 
Draw $\lambda_{k,k}$  independently 
from a pdf proportional to the following function  on the interval $[-r^k,r^k]$
	\begin{align} \label{god2}
	 \left(1+\frac{\lambda_{k,k}}{r^{k}}\right)^{ \lfloor \frac{k}{2} \rfloor } \left(1-\frac{\lambda_{k,k}}{r^{k}} \right)^{ \lfloor \frac{k-1}{2} \rfloor},
	\end{align}
where 
\begin{align} \label{god}
\lambda_{i, k} = \lambda_{i, k-1} + \frac{ \lambda_{k,k} \lambda_{k-i, k-1} }{r^{2k-2i}} \quad (  i=1,\ldots,k-1) .
\end{align}
\end{lemma}
\begin{proof}
We prove by induction.  The pdf of $\lambda_{1,1}$ is proportional to one. For $k>1$, suppose that the pdf of $[\lambda_{1,k-1}, \ldots, \lambda_{k-1,k-1}]^\T$ is proportional one inside $R_{k-1}(r)$ and zero elsewhere. 
Suppose that $\lambda_{k,k} \in [-r^k,r^k]$ and $\lambda_{1,k},\ldots, \lambda_{k-1,k}$ are determined by (\ref{god}). 
The Levinson-Durbin recursion in (\ref{god}) automatically enforces the stability constraint that $[\lambda_{1,k}, \ldots, \lambda_{k,k}]^\T $ falls inside $R_k(r)$.
The pdf of $[\lambda_{1,k}, \ldots, \lambda_{k,k}]^\T  $ can be computed as
\begin{align*}
& p(\lambda_{1,k}, \ldots, \lambda_{k,k}) 
= p(\lambda_{k,k})  p(\lambda_{1,k}, \ldots, \lambda_{k-1,k} \mid \lambda_{k,k}) \\
&= 
p(\lambda_{k,k}) p(\lambda_{1,k-1}, \ldots, \lambda_{k-1,k-1})  | J_k |^{-1} \\ 
&\propto p(\lambda_{k,k})   \left(1+\lambda_{k,k} / r^{k}\right)^{- \lfloor k/2 \rfloor } \left(1- \lambda_{k,k} / r^{k} \right)^{ - \lfloor (k-1)/2 \rfloor },
\end{align*}
where  $J_k = \det[\partial \lambda_{i,k} / \partial \lambda_{k-j,k-1}]_{1 \leq i,j\leq k-1}$ is the Jacobian from $\lambda_{i,k}$ to $\lambda_{k-i,k-1}$ ($i=1,\ldots,k-1$) taking $\lambda_{k,k}$ to be given. 
Therefore, if $p(\lambda_{k,k})$ is proportional to the value given by (\ref{god2}), the joint pdf of  $\lambda_{1,k}, \ldots, \lambda_{k,k}$ is proportional to one in $R_k(r)$ and zero elsewhere. 
\end{proof}
%
%
\begin{remark}
The technique presented in Lemma~\ref{lemma:distribution} can be equivalently formulated in a simple way summarized in the following lemma. 
The procedure is also described in Algorithm~\ref{algo:gen}.
\end{remark}
\begin{lemma} \label{lemma:uniform}
A sample of $[\lambda_{1,L} \ldots, \lambda_{L,L}]^\T$ that is uniformly distributed in $R_L(r)$ can be
generated by the recursion $\Lambda_0(z)=1, \Lambda_k (z) = z\Lambda_{k-1} (z) + r^k \alpha_k \overline{ \Lambda_{k-1} } (z/r^2)$,
where  $\alpha_k = 2 \beta_k - 1$ and $\beta_k \sim \text{Beta}(\lfloor k/2+1 \rfloor,\lfloor (k+1)/2 \rfloor), k=1,\ldots,L$ are independently generated.
\end{lemma}

\alglanguage{pseudocode}
\begin{algorithm*} 
\small
\caption{Generating a uniform sample $[\lambda_{1,L},\ldots,\lambda_{L,L}]^\T$ within $R_L(r)$}
\label{algo:gen}
\begin{algorithmic}[1]
\INPUT $L, r , \Lambda_0 (z) = 1$.
\OUTPUT $\Lambda_L(z) = z^L+ \sum_{\ell=1}^{L} \lambda_{\ell,L} z^{L-\ell} $.
  \For{ $k=1 \to L$ }
    \State Draw $\beta_k$ independently from the beta distribution $\beta_k \sim \text{Beta}(\lfloor k/2+1 \rfloor,\lfloor (k+1)/2 \rfloor)$
    \State  Let $\alpha_k = 2 \beta_k - 1$ and
        $\Lambda_k (z) = z\Lambda_{k-1} (z) + r^k \alpha_k \overline{ \Lambda_{k-1} } (z/r^2).$
  \EndFor
\end{algorithmic}
\end{algorithm*}

Fig.~\ref{randomPoly} illustrates the filters randomly generated from $R_2(r)$ with $r=0.6,0.8,1$. 
The centers of a two-clustering obtained using Algorithm~\ref{algo:cluster} are also shown in this figure. 
These centers are calculated based on the average of 20 random instances, each with 1000 samples.
Fig.~\ref{Ref_Curve} shows the reference curves for $r=0.6,0.8,1$ and $L=4$. 

\begin{figure}[t]
\centering
  \includegraphics[width=0.97\linewidth]{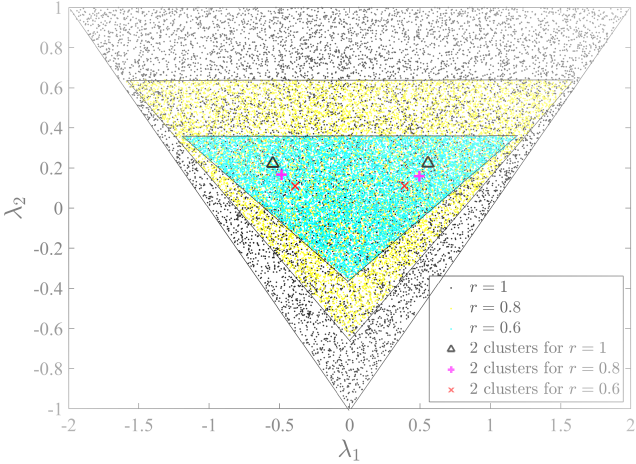}
  \captionof{figure}{10000 independent and uniformly distributed filters of $L=2$ 
 				 and the centers of two clusters, with $r=0.6,0.8,1$.}
  \label{randomPoly}
\hspace{0.2in}
\end{figure}

\begin{figure}[t]
  \centering
  \includegraphics[width=0.97\linewidth]{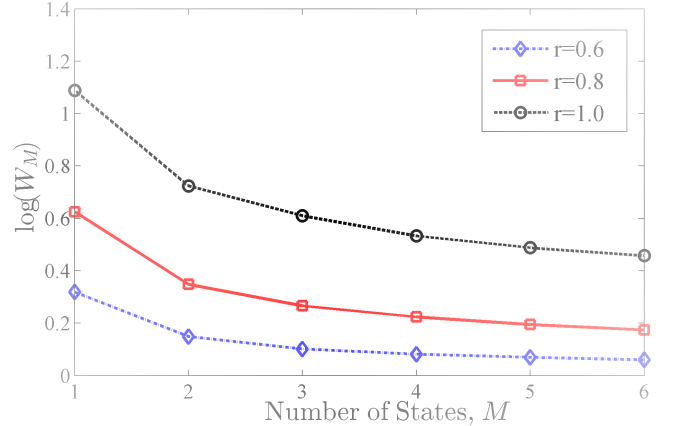}
  \vspace{0.05in}
  \captionof{figure}{The reference curves for $r=0.6,0.8,1$, $L=4$, which are obtained based on $Iter=32, F=1000$ (see Algorithm~\ref{algo:gap}).}
  \label{Ref_Curve}
\end{figure}

\section{Model} \label{sec:model}



A popular way to describe the switching behavior between different states is to assume that the transition between the states follows a first-order Markov process.
In this section, we adopt this assumption to formulate a parametric multi-state AR model for illustration purpose, even though the model selection criterion proposed in Section~\ref{sec:gap} is applicable to other multi-state AR models. 

\subsection{Notations and Formulations} \label{subsec:EM}


%
%
Let $S_m$ denote the set of data points $x^{(n)}$ that are generated from state $m$.
Suppose that $x^{(-L+1)}, \ldots, x^{(0)}$ are fixed and known.
Let $\bm{Z} = \{ \bm z^{(n)} \}_{n=1}^N$ and $\bm{Y} = \{\bm y^{(n)} \}_{n=1}^N$ be a sequence of missing (unobserved) indicators, where $\bm z^{(n)}$ is a $M\times M$ matrix, $\bm y^{(n)}$ is a $M\times 1$ vector, and
\begin{align*}
z^{(n)}_{m m'}&=\left\{
\begin{aligned}
1 & \textrm{ if $x^{(n-1)} \in  S_{m}$ and $x^{(n)} \in  S_{m'}$}, \\
0 & \textrm{ otherwise},
\end{aligned}
\right. \\
y^{(n)}_{m}&=\left\{
\begin{aligned}
1 & \textrm{ if $x^{(n)} \in S_{m}$}, \\
0 & \textrm{ otherwise}.
\end{aligned}
\right.
\end{align*}
Clearly, $\bm z^{(n)} = \bm y^{(n-1)} (\bm y^{(n)})^\T$.
We note that $\bm y^{(n)}$ is a binary vector of length $M$ containing a unique ``$1$'';  with a slight abuse of notation $y^{(n)}$ is the location of that ``$1$''.
We assume that $\{y^{(n)}\}_{n=1}^N$ is a Markov chain with  transition probability matrix $\bm T$, where $P(x^{(n)} \in S_m, x^{(n+1)} \in S_{m'}) = T_{m m'}$, and
$y^{(1)} $ is drawn from $\mathcal{M} (\alpha_1, \ldots,  \alpha_M)$, where $\mathcal{M}$ denotes the family  of multinomial distributions.
In other words, the assumed data generating process (given a fixed $M$) is:
\begin{align}
y^{(n)} &\sim
  \begin{cases}
   \mathcal{M}( \alpha_1, \ldots,  \alpha_M)  & \text{if } n=1 ,\\
   \mathcal{M}( T_{y^{(n-1)} 1}, \ldots,  T_{y^{(n-1)} M})        & \text{otherwise},
  \end{cases}
\\ X^{(n)} &\sim \mathcal{N}( - \bm \gamma_{y^{(n)}}^T \bm x^{(n)} , 
\sigma_{y^{(n)}}^2), \quad n =2,\ldots, N. \label{model}
\end{align}

Let $\Theta = \{ \bm \gamma_m, \sigma_m^2,  T_{mm'},  m,m'=1,\ldots,M \}$ be the set of unknown parameters to be estimated, 
where $\bm \gamma_m$ is of length $L+1$ (including the constant term).
Though computing the maximum-likelihood estimation (MLE) of the above probabilistic model (\ref{model}) is not tractable, it can be approximated by a local maximum via the EM algorithm \cite{dempster1977maximum}.
The EM algorithm produces a sequence of estimates by the recursive application of E-step and M-step to the complete log-likelihood until a predefined convergence criterion is achieved.
The complete log-likelihood can be written as 
\begin{align}
   \sum\limits_{n=1}^N \log p(x^{(n)} \mid \bm x^{(n)})
  &= \sum\limits_{n=1}^N \sum\limits_{m,m'=1}^N z^{(n)}_{m m'} \left( \log \left( \frac{T_{m m'}}{\sqrt{2\pi }\sigma_{m'}}\right) \right. \nonumber \\
  &\left. +  \frac{\left(x^{(n)} - \bm \gamma_{m'}^T \bm x^{(n)}\right)^2}{2 \sigma^2_{m'}}  \right)  . \label{loglik}
\end{align}
For brevity, we provide the EM  formulas below without derivation. 
In the \emph{E-step}, we obtain a function of unknown parameters by taking the expectation of (\ref{loglik}) with respect to the missing data $\bm Y$ and  $\bm Z$ given the most updated parameters,  
\begin{align} \label{eq5}
Q( \Theta \mid \bm X,  \Theta^{\textrm{old}})
  &= \sum\limits_{n=1}^N \sum\limits_{m,m'=1}^N w^{(n)}_{m m'} \left( \log \left( \frac{T_{m m'}}{\sqrt{2\pi }\sigma_{m'}}\right) \right. \nonumber \\ 
  &\left. +  \frac{\left(x^{(n)} - \bm \gamma_{m'}^T \bm x^{(n)}\right)^2}{2 \sigma^2_{m'}}  \right),
\end{align}
where
\begin{align}
w^{(n)}_{m m'}
&= E (z^{(n)}_{m m'} \mid   \Theta^{\textrm{old}} ) = P(y^{(n-1)} = m, y^{(n)} = m' \mid \bm X ) \label{EM_w} 
\end{align}
can be computed recursively. We note that the parameters involved in the right-hand side of (\ref{EM_w}) take values from the last update. 
In the \emph{M-step}, we use the coordinate ascent algorithm to obtain the following local maximum.
The ``old" superscriptions are omitted for brevity.
%
\begin{align}
\bm \gamma_m
&= - \left( \sum\limits_{n=1}^N  \sum\limits_{m'=1}^{M} w_{m' m}^{(n)}  \bm x^{(n)} (\bm x^{(n)})^\T \right)^{-1} \nonumber \\
 & \qquad \left( \sum\limits_{n=1}^N  \sum\limits_{m'=1}^{M}  w_{m' m}^{(n)}  x^{(n)} \bm x^{(n)}  \right) , \label{EM_gamma} \\
%
%
 \sigma_m^2 &= \frac{ \sum\limits_{n=1}^N  \sum\limits_{m'=1}^{M} w_{m' m}^{(n)} \left( x^{(n)} +  \bm \gamma_m^\T \bm x^{(n)} \right)^2  }{\sum\limits_{n=1}^N  \sum\limits_{m'=1}^{M} w_{m' m}^{(n)}},  \label{EM_sigma} \\ 
T_{m m'} &= \frac{ \sum\limits_{n=1}^N w_{m m'}^{(n)} }{ \sum\limits_{m'=1}^{M} \sum\limits_{n=1}^N  w_{m m'}^{(n)} }. \label{EM_T}
\end{align}


\subsection{Initialization of EM} \label{subsection:initialization}

The convergence speed of the EM algorithm strongly depends on the initialization and an improper initialization can cause it to converge to a local maximum which is far away from the global optimum.
A routine technique is to use multiple random initializations and choose the output with the largest likelihood \cite{mclachlan2004finite}, but this can be significantly time consuming. 
Here, we use a new initialization technique to get a fast and reliable convergence for the EM algorithm. 
This technique is based on the fact that for time series obtained in most practical areas, the self-transition probability of the states is usually close to one, i.e., $T_{mm} \approx 1$.
By adopting this assumption, we propose the initialization method in Algorithm~\ref{algo:EM}, which is shown empirically to produce more reliable and efficient EM results.
We note that the ``split'' style rule that appears in line 5 of Algorithm~\ref{algo:EM} is used elsewhere (e.g. s\cite{ueda2000smem}).

\alglanguage{pseudocode}
\begin{algorithm*}
\small
\caption{EM algorithm with the proposed initialization approach}
\label{algo:EM}
\begin{algorithmic}[1]
\INPUT $\bm{X} = \{ x^{(n)} \}_{n=1}^N$.
\OUTPUT  The initial parameters  $\hat{\Theta}_M = \left\{ \hat{\Gamma}_M = \{\hat{\bm \gamma}_m \}_{m=1}^M, \ \hat{\Sigma}_M = \{\hat{\sigma}^2_m\}_{m=1}^M , \  \hat{\bm T}_M = \left((\hat{T}_M)_{mm'}\right)_{m,m'=1}^M \right\}, \ M = 1, \ldots, {M_{\textrm{max}}}$. 
    \For {$M = 1 \to M_{\textrm{max}} $}
    	\For {$n = 1 \to N-N_0+1$ }
      	\State
      	Estimate the AR filter $\bm \hat{\bm \xi}_n$ and the noise variance $\hat{\sigma}^2_n$ from the sequence $\{x^{(n)},\ldots,x^{(n+N_0-1)}\}$ via the least squares 	method. 
    	 \EndFor
    	\State
      	Cluster $\hat{\bm \xi}_1,\ldots, \hat{\bm \xi}_{N-N_0+1}$ into $M$ cluster using $k$-means algorithm and obtain the centers $\hat{\bm \varrho}_1, \ldots, \hat{\bm \varrho}_M$ with the corresponding noise variances $\hat{\varsigma}_1^2, \ldots, \hat{\varsigma}_M^2$.
    Pick up such $\hat{\bm \varrho}_{k}$ ($1 \leq k \leq M$) that maximize the sum of Euclidean distances to $\hat{\bm \gamma}_1, \ldots, \hat{\bm \gamma}_{M-1}$.
        \If {$M>1$}
            	\State
     	 	Let $\hat{\Gamma}_M = \hat{\Gamma}_{M-1} \cup \hat{\bm \varrho}_k, \ \hat{\Sigma}_M  = \hat{\Sigma}_{M-1} \cup \hat{\varsigma}_k^2 , \ (\hat{T}_M)_{mm'} = 1/M \ (m,m' = 1, \ldots, M)$.
	\Else 
		\State  $\hat{\Gamma}_{1} = \hat{\bm \varrho}_k, \ \hat{\Sigma}_1  = \hat{\varsigma}_k^2 , \  \bm T_1 = 1$.
        \EndIf
     \State Run EM updates described in (\ref{EM_w})-(\ref{EM_T}) till certain stopping criterion is achieved.
    \EndFor
\end{algorithmic}
\end{algorithm*}


\section{Numerical Experiments} \label{sec:performance}

This section presents numerical results to evaluate the performance of the proposed technique.

%
%

\begin{figure}
\centering
  \includegraphics[width=0.97\linewidth]{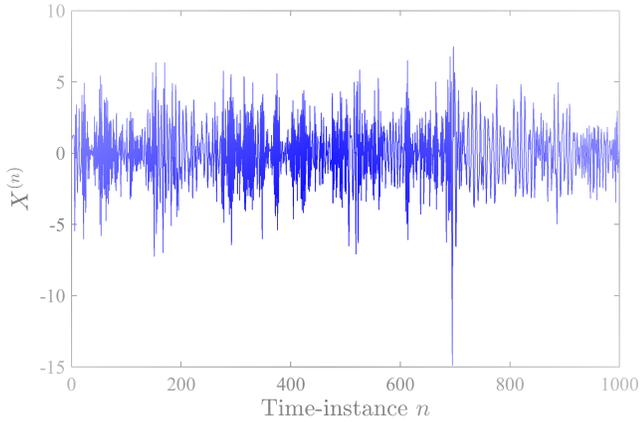}
  \vspace{-0.0in}
  \captionof{figure}{A random instance of multi-state AR time series: $L=4, M=3,  T_{mm}=0.98, \mu_m=0, \sigma_m=1, T_{mm'} = 0.01, m, \, m'=1,\ldots,3, m \neq m'$.}
  \label{Sample}
\end{figure}

\begin{figure}
  \centering
  \includegraphics[width=0.97\linewidth]{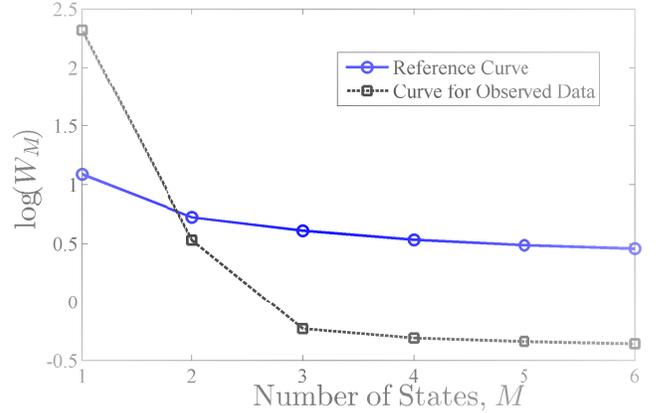}
  \captionof{figure}{The reference curves and the observed MSPE curve for the time series shown in Fig.~\ref{Sample}. The gap between the two curves is maximized at $M=3$. }
  \label{Ref_Curve_Observed}
\end{figure}

\begin{table*}
 \label{tab:results} 
    \vspace{-0.0 in}
    \begin{center}
        \includegraphics[width=5in]{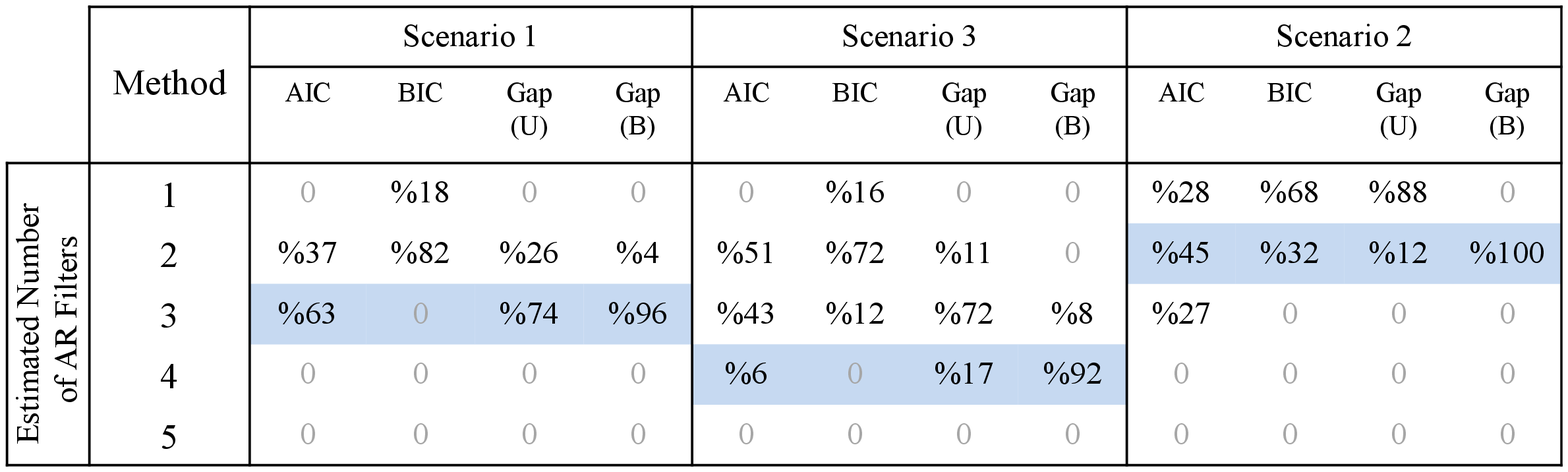}
    \end{center}
    \vspace{-0.1 in}
    \caption{The estimated number of AR filters for three different scenarios using AIC, BIC and Gap statistics (with the true number of filters for each scenario  highlighted)}
\end{table*}

Fig.~\ref{Sample} shows a time series generated from a 3-state AR model with $L=4$. 
The observed MSPE curve associated with the time series shown in Fig. 3 and the reference curve for $L=4$ are plotted in Fig.~\ref{Ref_Curve_Observed}.
The gap between the two curves is maximized at $M=3$. Thus, the selected $M$ is 3. 
In order to compare the performance of the proposed technique with those of AIC and BIC, we have generated synthetic time series under three different scenarios and apply each technique on those data to estimate the number of states. The three scenarios are as follows:\\
\noindent Scenario 1: $(L, M, r)=(4, 3,1)$ ; \\ 
\noindent Scenario 2: $(L, M, r)= (1,4,0.8)$; \\ 
\noindent Scenario 3: $(L, M, r)=(2,2,0.6)$.\\
\noindent For each scenario, $100$ instances of multi-state AR time series of length $N=1000$ are independently generated, each of which consists of $M$ autoregressive filters which are uniformly drawn from the $R_L(r)$ space. For each AR, the mean is uniformly generated from $[-4,4]$ and the variance is assumed to be $1$.
The transition matrix is considered to be $T_{mm}=0.98, T_{mm'} = 0.02/(M-1)$ for $m,m'=1,\ldots,M, m\neq m'$.
For each instance, the model parameters for each fixed $M=1,\ldots,M_{\text{max}}$ are estimated using EM algorithm, where $M_{\text{max}}=6$.
%
Table 1 shows the estimated number of AR filters using AIC, BIC and Gap statistics, where two types of Gap statistics are used to estimate the number of states. In the first type, denoted by Gap (U), the reference curve is generated from sample AR filters that have roots inside the unit circle, and is therefore independent of the data. In the second form of the Gap statistics, represented by Gap (B), the sample filters are restricted to have roots inside a circle with radius $r$, where $r$ is calculated from the data based on Algorithm~\ref{algo:gap}.
According to these results, Gap (B) outperforms AIC and BIC in all three scenarios, and it gives a better estimate of the number of states compared with Gap (U) since it is adaptive to the data.

\section{Conclusions} \label{sec:conclusion}
In this paper we proposed a model selection technique to estimate the number of states in a time series. 
The proposed approach, referred to as the Gap statistics, uses a reference curve to check whether adding a new state significantly decreases the prediction error. 
The reference curve is calculated by clustering uniformly generated stable AR filters. 
Numerical results show that the performance of the proposed model selection criterion surpasses those of AIC and BIC.



\section*{Acknowledgments}
This work is supported by Defense Advanced Research Projects Agency (DARPA) under grant number W911NF-14-1-0508.



\bibliographystyle{unsrt}
\balance
\bibliography{ARMA,HMM,Clustering}

\end{document}